\documentclass[a4paper,fleqn]{cas-dc}
\usepackage{graphicx}
\usepackage{amsmath,amssymb}
\usepackage{booktabs}
\usepackage{hyperref}
\newenvironment{proof}{\par\noindent\textit{Proof.}\ }{\hfill$\square$\par}
\newtheorem{proposition}{Proposition}

\def\tsc#1{\csdef{#1}{\textsc{\lowercase{#1}}\xspace}}
\tsc{WGM}
\tsc{QE}
\begin{document}
\let\WriteBookmarks\relax
\def\floatpagepagefraction{1}
\def\textpagefraction{.001}
%\interfootnotelinepenalty=10000
% ===============================
% Short title and short author
% ===============================
\shorttitle{STAP: Shuffle-Tokenized App Predictor}
\shortauthors{C. Fan et al.}
% ===============================
% Full title
% ===============================
\title [mode = title]{STAP: A Shuffle-Tokenized App Predictor with Ultra Long Context for Vocabulary-Free Mobile App Prediction}
% ===============================
% Author 1
% ===============================
\author[1]{Chengyu Fan}[orcid=0009-0009-2413-295X]             
\ead{fancy6690@mail.ustc.edu.cn}
% ===============================
% Author 2
% ===============================
\author[2]{Hang Liu}[orcid=0000-0002-4643-1231]
\cormark[1]  
\ead{hliu01@ustc.edu.cn}
% ===============================
% Affiliations
% ===============================
\affiliation[1]{organization={School of Nuclear Science and Technology, University of Science and Technology of China},
    addressline={96 Jinzhai Road},
    city={Hefei},
    postcode={230026},
    state={Anhui},
    country={China}}

\affiliation[2]{organization={Department of Statistics and Finance, University of Science and Technology of China},
    addressline={96 Jinzhai Road},
    city={Hefei},
    postcode={230026},
    state={Anhui},
    country={China}}

% ===============================
% Corresponding author text
% ===============================
\cortext[1]{Corresponding author: Hang Liu, hliu01@ustc.edu.cn}

% ===============================
% Abstract
% ===============================
\begin{abstract}
Predicting the next mobile application a user will launch is essential for intelligent device resource management and proactive assistance.
Existing models rely on fixed app vocabularies, which prevents them from generalizing across different app ecosystems. Many also depend on user-specific knowledge, which complicates deployment in cold start scenarios.
We propose STAP, a Transformer-based model that eliminates the need for a fixed vocabulary.
STAP replaces true app identities with randomly reassigned virtual indices via a shuffle mechanism, and compensates for discarded semantic information by processing behavioral sequences with an ultra-long context design.
A theoretical analysis shows that, given a sufficiently long context, the predicted distribution converges to the correct one despite the anonymity of the mapping.
Experiments on two datasets from different continents demonstrate that STAP achieves strong cross-dataset zero-shot prediction accuracy---a setting where all existing fixed-vocabulary methods are inherently inapplicable---while its cold start performance within each dataset remains competitive with leading models. Furthermore, we introduce a deployment strategy that enables the model to retain a sufficiently long context during continuous inference while keeping latency within acceptable bounds.

\end{abstract}
% ===============================
% Graphical abstract (optional)
% ===============================
%\begin{graphicalabstract}
%\includegraphics{graphical_abstract.pdf}
%\end{graphicalabstract}
% ===============================
% Research highlights
% ===============================
% \begin{highlights}
% \item Shuffle mechanism decouples prediction from fixed app vocabularies.
% \item Ultra-long context recovers behavioral semantics lost during shuffling.
% \item ISWI guarantees deep context while avoiding costly prefills.
% \item Zero-shot cross-dataset mobile app prediction demonstrated.
% \item Theoretical justification for vocabulary-free learning under random indexing.
% \end{highlights}
% ===============================
% Keywords
% ===============================
\begin{keywords}
Mobile app prediction \sep vocabulary free \sep Transformer \sep shuffle mechanism \sep long context \sep cross dataset generalization
\end{keywords}

\maketitle

\section{Introduction}\label{sec:introducion}

\subsection{Literature review on app prediction}
% Today, mobile phones are becoming more and more essential in our daily life. To make the user experience even better, it is very useful to predict which app a user will open next \cite{introduction1}. The prediction results can be used for many purposes, such as: recommending apps so users can open them more easily and preloading apps to reduce startup waiting time \cite{show_app}\cite{appflow}.

As mobile devices pervade our daily lives, accurately predicting the next app a user is likely to launch is crucial for enhancing user experience \cite{introduction1}. Such predictions can serve various purposes, such as recommending apps for easier access and preloading apps to reduce startup latency \cite{show_app}\cite{appflow}.

However, app prediction remains a challenging task. Early explorations either underperform since they fail to leverage complex temporal and multimodal information, or rely heavily on hand-crafted features, which limits them to a single dataset \cite{early_1, early_2, early_3, early_4, knn, mk}. Consequently, these methods offer limited value for real-world applications.

With the development of deep learning techniques, some studies attempt to introduce RNNs, attention mechanisms, graph networks, and similar architectures to capture more complex patterns \cite{DeepApp, DeepPattern,App2Vec,SAGCN}. However, since these methods rely on user-specific information, they cannot be applied to cold-start settings---i.e., training on a subset of users and testing on unseen users.

More recent work has started to incorporate large language models or design sophisticated temporal enhancement modules \cite{CoSEM,NeuSA,MAPLE,TGT}. These efforts endow models with stronger generalization capabilities, enabling them to address the challenging cold-start task. However, due to their dependence on a fixed app vocabulary, the cold-start ability of these models is limited to a single dataset: they cannot achieve cross-dataset generalization or require fine-tuning to do so.

Moreover, real-world deployment poses a challenge rarely captured by standard benchmarks: new apps constantly appear after training. While existing cold-start methods can handle new users, they still assume a fixed app vocabulary and thus cannot handle apps unseen during training.
% Open datasets often exacerbate this issue---privacy concerns lead to anonymization of app names  (e.g., the Tsinghua App Usage Dataset), \textcolor{blue}{and differences in time or location shift the popular app set. Rewrite this sentence} Therefore, cross-dataset zero-shot transfer, where the model generalizes to a new app ecosystem without any fixed vocabulary, is a critical step from laboratory research to practical deployment.
Open datasets often exacerbate this issue---privacy concerns lead to anonymization of app names  (e.g., the Tsinghua App Usage Dataset), and differences in collection time and location can lead to entirely different popular-app sets across datasets. Therefore, cross-dataset zero-shot transfer, where the model generalizes to a new app ecosystem without any fixed vocabulary, is a critical step from laboratory research to practical deployment.

\subsection{Our contributions}

We propose {\it Shuffle-Tokenized App Predictor} (STAP), a Transformer‑based model that operates without a fixed app vocabulary. Our key contributions are as follows. 

\begin{enumerate}
    \item[(i)] The model introduces a \textbf{shuffle mechanism} that randomly maps each original app to to a virtual index per user and per epoch, thereby completely decoupling the model from any specific app set. This enables zero‑shot transfer across disjoint app ecosystems without retraining or fine‑tuning.
    \item[(ii)] The model incorporates an \textbf{ultra‑long context design} that compensates for the loss of app identity information by leveraging extended temporal and structural patterns. We provide a theoretical justification showing that with sufficiently long context, the model converges to the correct predictions despite the anonymous mapping.
    \item[(iii)] Furthermore, we design an \textbf{interleaved semi‑window inference (ISWI)} strategy that enables the model to retain a sufficiently long context during continuous inference while maintaining an acceptable latency.    
    \item[(iv)] Extensive experiments on two datasets from different continents demonstrate that STAP not only achieves cross-dataset zero-shot transfer but also remains competitive with state‑of‑the‑art models.
\end{enumerate}

The remainder of this paper is organized as follows. Section~\ref{sec:methodology} formulates the problem, introduces the core mechanisms (the shuffle mechanism and the ultra-long context strategy), the model architecture and ISWI deployment strategy. Section~\ref{sec:experiments} presents comprehensive numerical results, including cross‑dataset and in‑dataset performance, ablation and sensitivity studies, as well as measurements of inference latency and memory usage. Section~\ref{sec:discussion} provides a theoretical justification for the shuffle mechanism and the role of long context, reports distinctive training dynamics, and discusses limitations and future work. Section~\ref{sec:conclusion} concludes the paper.

\section{Methodology}\label{sec:methodology}
\subsection{Problem Formulation}\label{sec:fomulation}

% \textcolor{blue}{Be careful with the notation: I replaced $k$ with $s$ in this section, since $k$ is used in Section 3.1.3.}

We aim to predict the next application a user will open. Let $\mathcal{A} = \{a_1, a_2, \dots, a_M\}$ denote the finite set of mobile applications. A user’s interaction history is represented as a chronological sequence of app usage events. Each event is defined as a quadruple $\mathrm{E}_i = (A_i, C_i, T_i, H_i)$, where:
\begin{itemize}
    \item $A_i \in \mathcal{A}$ denotes the unique App ID;
    \item $C_i \in \{0, 1\}$ indicates the action type (0 for \textit{Close}, 1 for \textit{Open});
    \item $T_i \in \mathbb{R}$ represents the absolute timestamp (measured in minutes);
    \item $H_i \in [0, 24)$ denotes the hour of the day.
\end{itemize}
% A user's interaction can be treated as an infinite sequence of events $\mathrm{E}_{1:\infty}= (\mathrm{E}_1, \mathrm{E}_2, \dots)$. Given $\mathrm{E}_{1:\infty}$ and a prefix sequence $\mathrm{E}_{1:t} = (\mathrm{E}_1, \mathrm{E}_2, \dots, \mathrm{E}_t)$ up to time $t$ , the goal is to predict the next app $A_{s} \in \mathcal{A}$ that the user will open for $s = \min\{i>t \mid C_i = 1\} $. 

A user's interaction can be treated as an random infinite sequence of events $\mathrm{E}_{1:\infty}= (\mathrm{E}_1, \mathrm{E}_2, \dots)$. Given $\mathrm{E}_{1:\infty}$ and a prefix sequence $\mathrm{E}_{1:t} = (\mathrm{E}_1, \mathrm{E}_2, \dots, \mathrm{E}_t)$ up to time $t$ , the goal is to predict the next app $A_{s} \in \mathcal{A}$ that the user will open for $s = \min\{i>t \mid C_i = 1\} $.

% This is a multi-class classification problem with $M = |\mathcal{A}|$ classes. The prediction model is a function $f : \mathrm{E}^* \rightarrow \mathbb{R}^M$ mapping any prefix $\mathrm{E}_{1:t} \in \mathrm{E}^*$ \textcolor{blue}{$\mathrm{E}^*$ should be defined} to a vector of scores (logits) $\mathbf{z}_t = f(\mathbf{E}_{1:t} \mid \Theta)$, with $\Theta$ denoting the model parameter. The predicted probability distribution over apps is obtained via the softmax function:
Let $\mathrm{E}^*$ be the set of all finite sequences of events (all possible prefixes). 
This is a multi-class classification problem with $M = |\mathcal{A}|$ classes. The prediction model is a function $f : \mathrm{E}^* \rightarrow \mathbb{R}^M$ mapping any prefix $\mathrm{E}_{1:t} \in \mathrm{E}^*$ , with $\Theta$ denoting the model parameter. The predicted probability distribution over apps is obtained via the softmax function:
% $\hat{\mathbf{p}} = \operatorname{softmax}(\mathbf{z_t}) \in \Delta^{M-1}$. \textcolor{blue}{since $\Delta$ is not used elsewhere, simply write $\in (0, 1)$. Do you mean $M$, rather than $M-1$?} Finally, the predicted app is the one with the highest probability, i.e.,
%$\hat{A}_{k} = \arg\max_i \hat{\mathbf{p}}_i$ 
% \textcolor{blue}{$$\hat{A}_{s} = \arg\max_{i \in \{1, \ldots, M\}} \hat{\mathbf{p}}_i.$$}
$\hat{\mathbf{p}} = \operatorname{softmax}(\mathbf{z_t})$, where $\hat{\mathbf{p}} = (\hat{{p}}_1, \ldots, \hat{{p}}_M) \in (0, 1)^M$ is a probability vector over $M$ classes. Finally, the predicted app is the one with the highest probability, i.e.,
%$\hat{A}_{k} = \arg\max_i \hat{\mathbf{p}}_i$ 
$$\hat{A}_{s} = \arg\max_{i \in \{1, \ldots, M\}} \hat{{p}}_i.$$

\subsection{Core Mechanisms}\label{sec:mechanism}
Our model relies on two complementary mechanisms that together eliminate the need for a fixed app vocabulary:
\begin{enumerate}
    \item The \textbf{shuffle mechanism} (Section~\ref{sec:shuffle}) removes the dependency on a specific app set, enabling cross-dataset transfer, and simultaneously acts as a strong regularizer that prevents the model from overfitting to noise under ultra-long contexts (see ablation study in Section~\ref{sec:ablation}).
    \item The \textbf{ultra-long context} (Section~\ref{sec:longcontext}) compensates for the information discarded by the shuffle mechanism, allowing the model to recover implicit behavioral cues from extended temporal patterns (see sensitivity study in Section~\ref{sec:sensitivity} and theoretical justification in Section~\ref{sec:theoretical}).
\end{enumerate}

\subsubsection{Shuffle Mechanism}\label{sec:shuffle}

\begin{figure*}[t]
    \centering
    \includegraphics[width=0.95\textwidth]{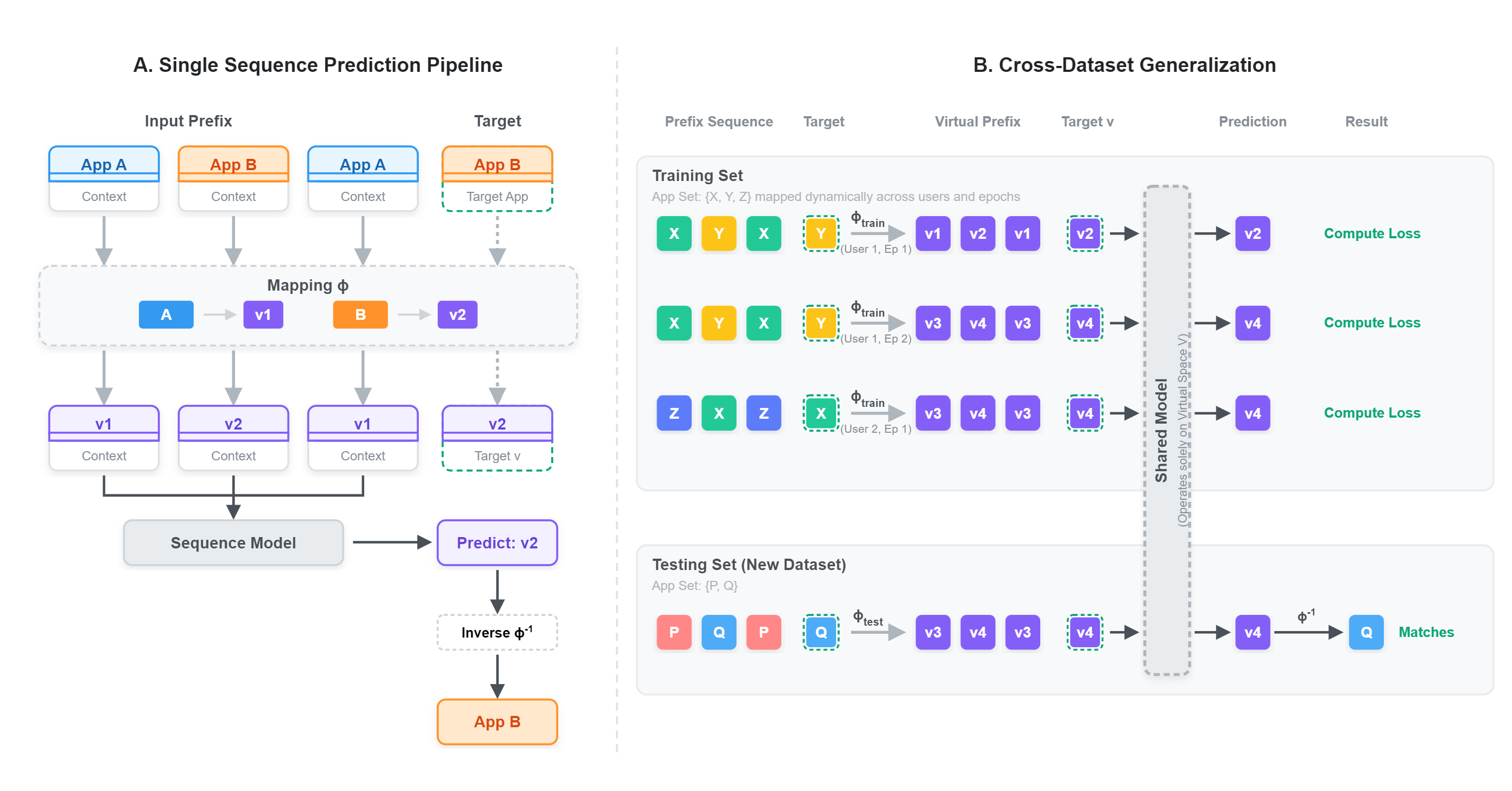}
    \caption{Schematic of the shuffle mechanism. (A)~Single sequence pipeline: real app IDs in the input sequence are mapped to virtual indices via a stochastic injective mapping $\phi$, which is independently sampled per user and epoch. The model processes and predicts entirely in the virtual vocabulary $\mathcal{V}$, and the real app is recovered via $\phi^{-1}$. (B)~Cross-dataset generalization: because $\phi$ is random and independent of any app-specific semantics, sequences from entirely unseen apps (e.g., $\{P,Q\}$) can be handled by the same virtual vocabulary, enabling zero-shot transfer.}
    \label{fig:shuffle_architecture}
\end{figure*}

Let $\mathcal{V} = \{0, 1, \dots, V-1\}$ be a fixed-size virtual vocabulary. Let $\mathcal{A}(\mathrm{E}_{1:\infty}) = \{A_i\mid i\in \mathbb{N}^+\}$. We require $V \ge |\mathcal{A}(\mathrm{E}_{1:\infty})|$ for every user retained in the dataset. In practice, we set $V = 200$, which covers the vast majority of users; the few users with more than $200$ distinct apps are treated as extreme outliers and discarded (see Section~\ref{sec:dataset}).

For \textbf{each user} and \textbf{each training epoch}, we independently sample an \textbf{injective mapping} $\phi: \mathcal{A}(\mathrm{E}_{1:\infty}) \rightarrow \mathcal{V}$ that assigns a unique virtual index to every app observed in that user's history. The model then predict the next virtual index $v_s = \phi(A_{s})$ based on the previous processed events $\mathrm{E}_{1:t}'$ where $\mathrm{E}_{1:t}' = (\mathrm{E}'_1, \mathrm{E}_2', \dots, \mathrm{E}_t')$ , with $\mathrm{E}_i' = (v_i, C_i, T_i, H_i)$ and $v_i = \phi(A_i)$~(See Fig. \ref{fig:shuffle_architecture}~A). In other words, the model becomes a function $f' : \mathrm{E'}^* \rightarrow \mathbb{R}^V$ that maps any prefix $\mathrm{E'}_{1:t} \in \mathrm{E'}^*$ to a vector of vitural scores $\mathbf{z}_t' = f'(\mathbf{E}_{1:t}' \mid \Theta)$. The final prediction of the real app becomes:
\begin{equation}
    \begin{aligned}
    \hat{\mathbf{p}}' &= \operatorname{softmax}(\mathbf{z}_t') \\
    \hat{A}_s &= \phi^{-1}\Big[\arg\max_i \{\hat{{p}}'_i\mid \exists a \in \mathcal{A}(\mathrm{E}_{1:\infty}) \, \text{s.t.} \, \phi(a) = i\}\Big],
    \end{aligned}
\end{equation}
with $\hat{\mathbf{p}}' = (\hat{{p}}'_1, \ldots, \hat{{p}}'_{|V|})$.

Because $\phi$ is a random injective map independent of any app-specific semantics, and because the model operates entirely on the virtual vocabulary $\mathcal{V}$, neither component relies on the true app identities. This is what makes cross-dataset generalization possible: the whole pipeline is decoupled from the fixed app set $\mathcal{A}$ used in training. (See Fig. \ref{fig:shuffle_architecture}~B)

\subsubsection{Ultra-long Context}
\label{sec:longcontext}

Conventional app usage prediction methods typically condition prediction on only a short recent window of length $L$, i.e., $f(\mathrm{E}_{1:t}) \approx f(\mathrm{E}_{t':t})$ with $t' = \max(1, t-L+1)$ \cite{CoSEM, MAPLE, TGT}. 
However, applying the shuffle mechanism with such a limited context would discard nearly all static semantic signals, since the random mapping $\phi$ removes app identity information from the input.

To compensate for this loss, we adopt an ultra-long context strategy, setting $L$ to be a sufficiently large value. 
A theoretical justification is provided in Section~\ref{sec:theoretical}, where we show that a sufficiently long context allows the model to converge to the correct predictions despite the anonymity of the mapping. 
Our sensitivity analysis (Section~\ref{sec:sensitivity}) further confirms that increasing the context length consistently improves prediction accuracy, with gains gradually saturating at larger values. 
Based on these experiments, we set $L = 4096$ as the default context length for all subsequent experiments.

\subsection{Architecture}\label{sec:architecture}

\begin{figure*}[t] 
    \centering
    \includegraphics[width=1.0\textwidth]{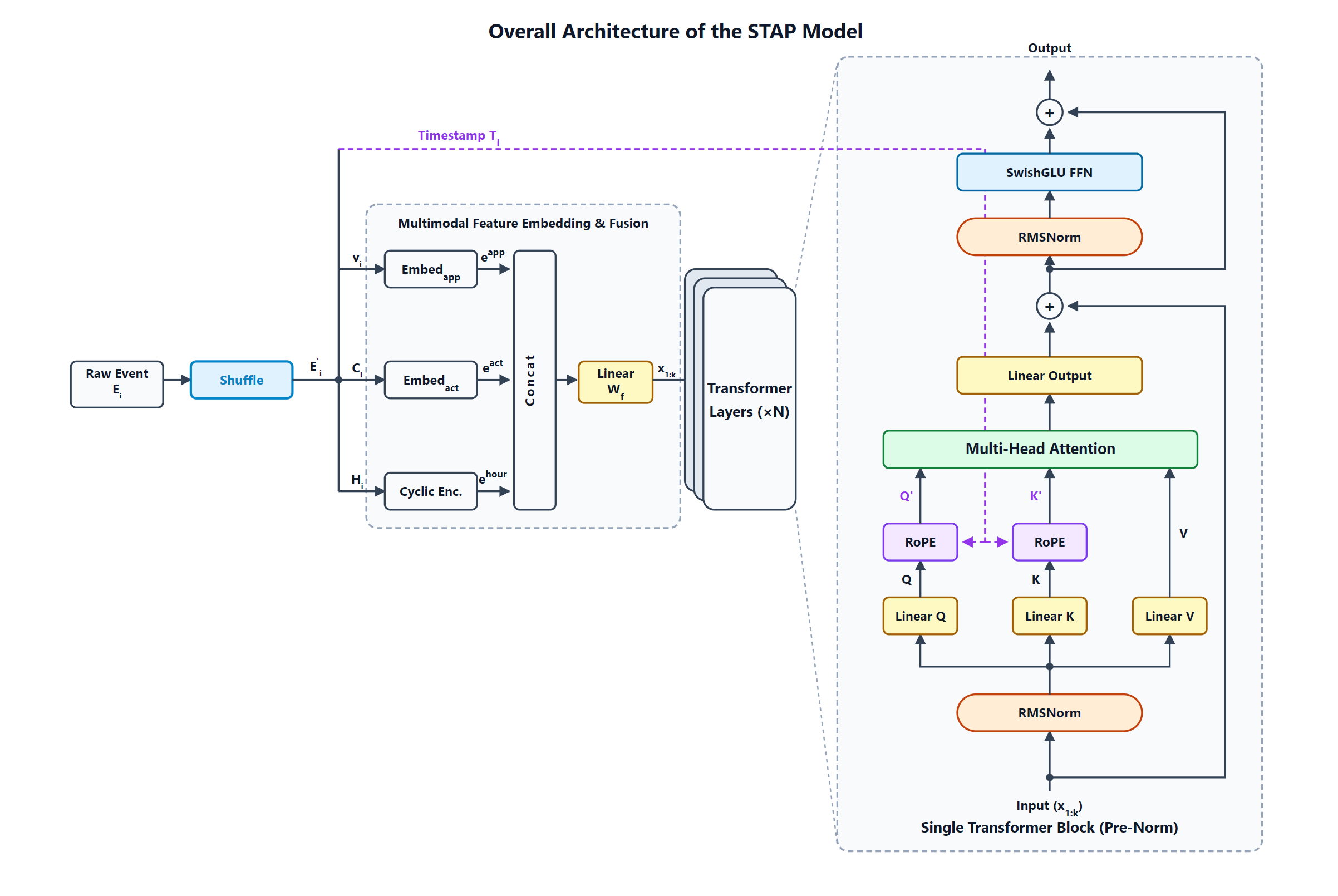} 
    \caption{The overall architecture of the STAP model. (Left) The input processing pipeline featuring the shuffle mechanism and multimodal feature fusion. (Middle) The backbone consisting of $N$ Transformer layers. (Right) The detailed structure of a single Transformer block, highlighting the Pre-Norm configuration, RMSNorm, SwishGLU activation, and the injection of absolute timestamps $T_i$ into the Rotary Positional Embedding (RoPE) module.}
    \label{fig:transformer_arch} 
\end{figure*}
The STAP model follows a modern Transformer-based architecture to ensure stable training and high capacity \cite{llama}. Specifically, we employ Pre-Norm residual connections \cite{prenorm}, RMSNorm for normalization \cite{rmsnorm}, SwishGLU activation \cite{swishglu}, and Rotary Positional Embeddings  \cite{rope}. The overall architecture is illustrated in Fig. \ref{fig:transformer_arch}.

\subsubsection{Multimodal Feature Embedding and Fusion}\label{sec:multimodal}
To capture the multidimensional nature of app usage events, we design a fusion layer that projects heterogeneous features into a unified $d$-dimensional latent space. For each event $\mathrm{E}_i$, the embedding process is as follows.

\textbf{Categorical features.} The virtual app index $v_i = \phi(A_i)$ and the action $C_i$ are embedded via two lookup tables:
$\mathbf{e}_i^{\mathrm{app}} = \mathrm{Embed}_{\mathrm{app}}(v_i)$,
$\mathbf{e}_i^{\mathrm{act}} = \mathrm{Embed}_{\mathrm{act}}(C_i)$.

\textbf{Cyclic temporal features.} To capture daily periodicity, the hour-of-day $H_i \in [0,24)$ is encoded as
$\mathbf{h}_i = [\,\sin(\tfrac{2\pi H_i}{24}),\; \cos(\tfrac{2\pi H_i}{24})\,]^{\top}$,
then projected to dimension $d$: $\mathbf{e}_i^{\mathrm{hour}} = \mathbf{W}_h\, \mathbf{h}_i + \mathbf{b}_h$.

The final input embedding $\mathbf{x}_i \in \mathbb{R}^d$ is obtained by concatenating the three vectors and passing them through a fusion linear layer:
\begin{equation}
    \mathbf{x}_i = \mathbf{W}_f \cdot [\mathbf{e}_i^{\rm app} ; \mathbf{e}_i^{\rm act} ; \mathbf{e}_i^{\rm hour}] + \mathbf{b}_f
\end{equation}
where $[\cdot ; \cdot]$ denotes the concatenation operation, and $\mathbf{W}_f \in \mathbb{R}^{d \times 3d}$ projects the fused features back to the model dimension $d$.

\subsubsection{Time-Aware Positional Encoding via RoPE}\label{sec:rope}

In standard RoPE, the rotation angle for query and key vectors at position $i$ is $i \cdot \Theta$, where $\Theta = \{\theta_k = b^{-2k/d}\}_{k=0}^{d/2-1}$ and the base $b$ is typically set to $10000$ \cite{rope}.
In our case, the absolute timestamps $T_i$ (measured in minutes) can span tens of thousands, so a larger base is needed to prevent the rotation angles from growing too rapidly, which would degrade the model's ability to distinguish nearby events.
We set $b = 100000$ and use the raw timestamp $T_i$ directly as the position signal.
The RoPE transformation becomes
\begin{equation}
    \tilde{\mathbf{q}}_i = \mathbf{R}_{T_i} \mathbf{q}_i,\quad \tilde{\mathbf{k}}_j = \mathbf{R}_{T_j} \mathbf{k}_j,
\end{equation}
where $\mathbf{R}_t$ is a block-diagonal rotation matrix with angles $t \cdot \theta_k$.
Consequently, the attention score $\tilde{\mathbf{q}}_i^{\!\top}\tilde{\mathbf{k}}_j$ depends only on the relative time difference $T_i - T_j$, enabling the model to perceive temporal distance naturally.

\subsubsection{Transformer Backbone and Training Objective}\label{sec:transformer}
The fused embeddings are fed into $N$ Transformer layers, each composed of a multi-head attention layer and a SwishGLU feed-forward network, with RMSNorm applied before each sub-layer (Pre-Norm). The output hidden state is projected to a distribution over the virtual vocabulary $\mathcal{V}$ via a linear layer and softmax. We train the model with the standard cross-entropy loss, where the target is the virtual index $v_{k}$; during inference, the predicted virtual index is mapped back to a real app using $\phi^{-1}$ as detailed in Section~\ref{sec:shuffle}.

\subsection{Deployment Strategy} \label{sec:iswi}
\begin{figure*}[t]
    \centering
    \includegraphics[scale=0.45]{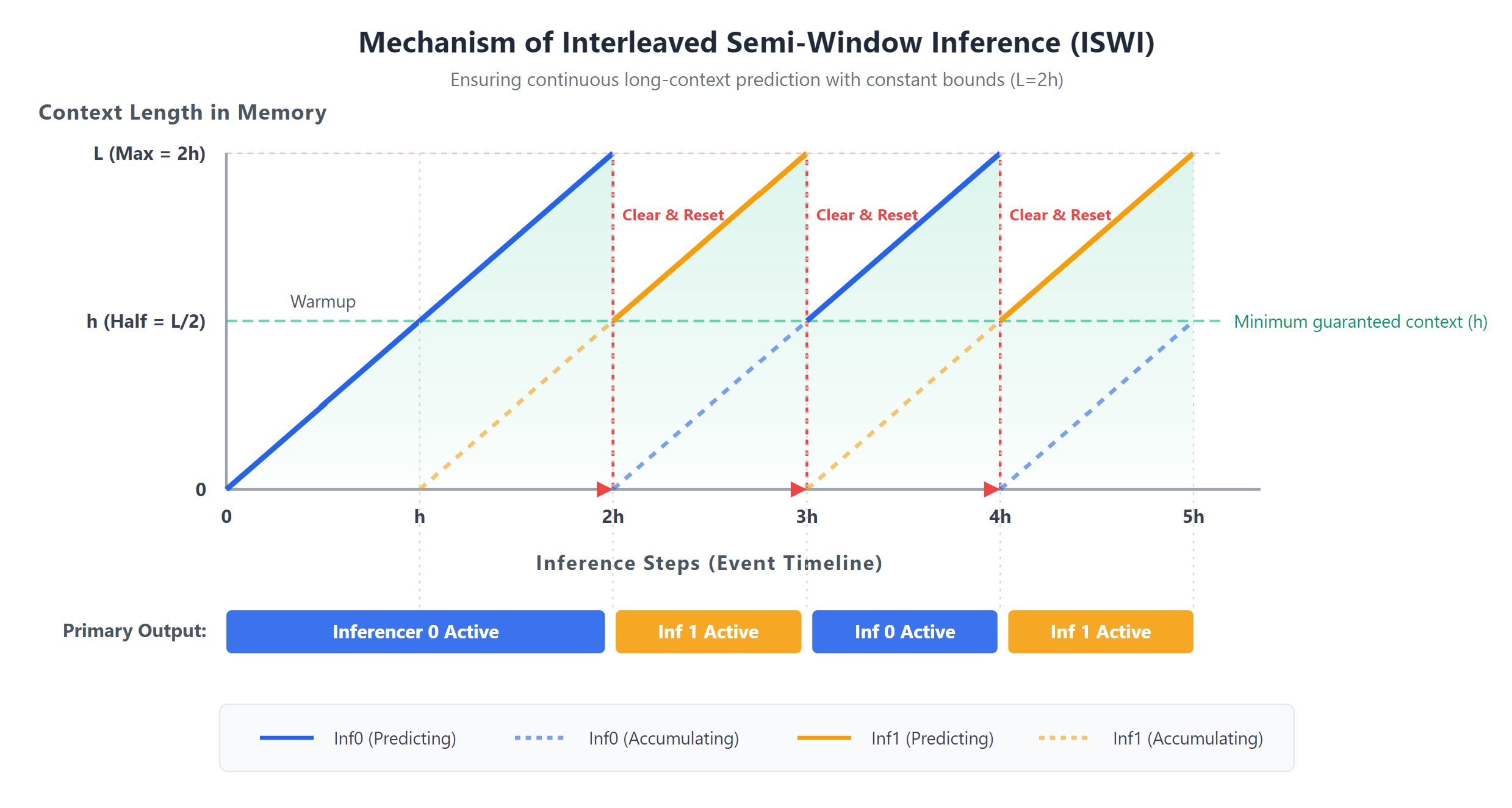}
    \caption{Illustration of the ISWI strategy. By maintaining two overlapping inference instances ($\text{Inf}_0$ and $\text{Inf}_1$), the system guarantees a minimum historical context of $h=L/2$ at any given prediction step (green shaded area), effectively eliminating the "cold-start" performance drop caused by periodic buffer resets. Solid lines indicate the active predicting instance, while dashed lines represent background context accumulation.}
    \label{fig:iswi}
\end{figure*}

% \begin{enumerate}
%     \item \textbf{Need for long context.} As established in Sections~\ref{sec:theoretical}, \ref{sec:longcontext} and \ref{sec:sensitivity}, the model relies on a sufficiently long history to recover the behavioural information masked by shuffling.
%     \item \textbf{Context cannot grow unboundedly.} The model is trained with a maximum context of $L$ events (e.g., $L=4096$), so inference beyond $L$ is out of distribution. Even if one ignores this training limit, the memory and per-step latency of a KV cache grow with its length, making unbounded growth infeasible on devices with limited resources.
%     \item \textbf{Restart cost.} When the cache reaches $L$, simply clearing it and starting a new window either discards too much context (cold-start accuracy drop) or, if the new window is long, incurs an $\mathcal{O}(L^2)$ prefill that is too slow.
% \end{enumerate}

Deploying a model with an ultra‑long context involves a trade‑off between prediction accuracy and inference cost. On one hand, as illustruated in Sections~\ref{sec:theoretical}, \ref{sec:longcontext} and \ref{sec:sensitivity}, the model relies on a sufficiently long history to recover the behavioural information masked by shuffling. On the other hand, the \textbf{context length cannot grow unboundedly}: The model is trained with a maximum context of $L$ events (e.g., $L=4096$), so inference beyond $L$ is out of distribution. Even if one ignores this training limit, the memory and per-step latency of a KV cache grow with its length, making unbounded growth infeasible on devices with limited resources.

When the cache reaches $L$, a straightforward approach is to reset it. However, simply clearing the cache and starting a new window discards too much context, leading to a sharp drop in accuracy (a cold-start problem). Conversely, prefilling a new long window from scratch incurs an \(\mathcal{O}(L^2)\) computational cost.

To resolve this conflict, we propose a \textbf{Interleaved Semi-Window Inference (ISWI)} strategy, which maintains two overlapping inference instances, $\text{Inf}_0$ and $\text{Inf}_1$, to guarantee a minimum historical context of $h = L/2$ at every prediction step (see Fig.~\ref{fig:iswi}). The prediction process is divided into stages, where step $T$ (starting from $0$) belongs to stage $\lfloor T/h \rfloor$. In stage $0$, only $\text{Inf}_0$ receives the new event and maintains the KV cache. For any stage $i>0$, both instances receive the new event. When $i$ is odd, $\text{Inf}_0$ is active and makes the predictions, while $\text{Inf}_1$ passively accumulates context; at the end of an odd stage, $\text{Inf}_0$ clears its KV cache. Symmetrically, when $i$ is even, $\text{Inf}_1$ takes over the active role (performing the same function as $\text{Inf}_0$ in odd stages) and makes predictions, while $\text{Inf}_0$ passively accumulates context; at the end of an even stage, $\text{Inf}_1$ clears its KV cache.

This design keeps both memory usage and per-step latency bounded: while the cache of one instance grows toward $L$ before being cleared, the other instance is always in a lighter accumulation phase, so the total resource consumption never exceeds that of two full caches. Quantitative measurements are reported in Section~~\ref{sec:exp_deploy}.

\section{Experiments}\label{sec:experiments}

% \textcolor{blue}{For notation consistency, you should replace $K$ with $L$.}
% I've replaced.

In this section, we conduct a comprehensive evaluation of the proposed framework. We first describe the experimental setup, including dataset, evaluation metrics, and implementation details. We then present a series of experiments targeting the core components of our model: (1) comparative performance analysis against leading models to assess predictive accuracy and robustness; (2) ablation and sensitivity studies to verify the necessity of the shuffling mechanism and the benefit of ultra-long historical context; and (3) measurement of the computational efficiency of the ISWI deployment strategy.

\subsection{Experimental Setup}\label{sec:setup}

\subsubsection{Dataset}\label{sec:dataset}

\begin{table}[htbp]
    \centering
    \caption{Dataset characteristics.}
    \label{tab:dataset}
    \begin{tabular}{lcc}
        \toprule
        \textbf{Dataset} & \textbf{Users} & \textbf{Apps} \\
        \midrule
        Tsinghua App Usage Dataset & 871 & 2000 \\
        LSapp                     & 292 & 87   \\
        \bottomrule
    \end{tabular}
\end{table}

We evaluate our approach on two widely used benchmarks: the Tsinghua App Usage Dataset \cite{TsinghuaApp} and the LSapp Dataset \cite{NeuSA} . Basic statistics are summarized in Table~\ref{tab:dataset}, where \emph{Users} indicates the number of valid users and \emph{Apps} is the number of distinct app categories in each dataset. The Tsinghua dataset assigns user IDs from 0 to 999, of which 871 actually appear and are retained; the rest correspond to no recorded usage. Since our model relies on ultra-long context (see Section~\ref{sec:longcontext}) and a virtual vocabulary (see Section~\ref{sec:shuffle}) whose size should exceed the number of app categories per user, we examine the per-user distribution of both total events\footnote{Each app usage contains one open event and one close event, so the number of events is twice the number of usage records.} and distinct apps across the whole dataset. These distributions are shown in Fig.~\ref{fig:distribution}. 

% \textcolor{blue}{Texts and numbers in Fig.~\ref{fig:distribution} are not clear. Could the resolution be improved? The font size should be larger.}

\begin{figure*}[t]
    \centering
    \includegraphics[width=\linewidth]{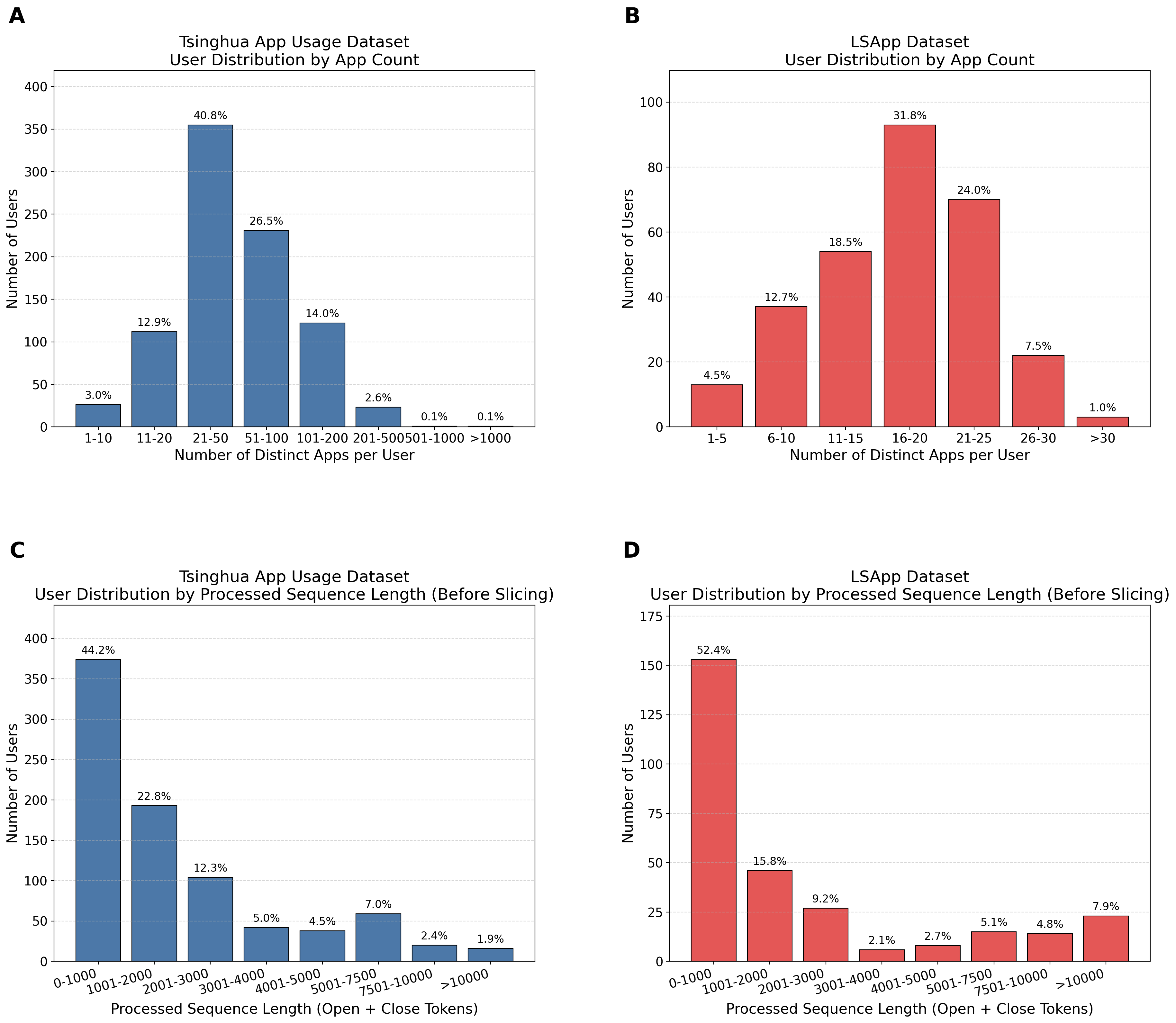} 
    \caption{Per-user distributions of distinct apps and processed event lengths across the Tsinghua and LSapp datasets. (A) and (B): number of distinct apps. (C) and (D): processed event length (twice the number of apps) before slicing.}
    \label{fig:distribution}
\end{figure*}

\noindent \textbf{Hyperparameter selection.} 
The virtual vocabulary size $|\mathcal{V}| = 200$ is chosen based on the per-user app count distributions (see Fig.~\ref{fig:distribution}): a vocabulary of 200 covers over 97\% of the users in the Tsinghua dataset and 100\% in the LSapp dataset. 
The maximum context length $L = 4096$ is determined by jointly considering these distributions and the sensitivity analysis in Section~\ref{sec:sensitivity}: most usage sequences in the two datasets contain fewer than 4096 events, and the performance gains from further increasing $L$ diminish beyond this point.

\subsubsection{Data processing}\label{subsec.process}

\begin{table*}[htbp]
    \centering
    \caption{Data split strategy (number of users).}
    \label{tab:dataset_split}
    \begin{tabular}{lcccccccl}
        \toprule
        \textbf{Experiment} & \multicolumn{3}{c}{\textbf{Tsinghua}} & \multicolumn{3}{c}{\textbf{LSapp}} & \multicolumn{2}{c}{\textbf{Description}} \\
        \cmidrule(lr){2-4} \cmidrule(lr){5-7} \cmidrule(lr){8-9}
        & Train & Val. & Test & Train & Val. & Test & Training on & Testing on \\
        \midrule
        Cross-dataset & 677 & 169 & --  & -- & -- & 292 & Tsinghua & LSapp \\
        Cross-dataset & --  & --  & 846 & 58 & 234 & -- & LSapp & Tsinghua \\
        In-dataset    & 592 & 169 & 85  & -- & -- & -- & Tsinghua & Tsinghua \\
        In-dataset    & --  & --  & --  & 204 & 29 & 58 & LSapp & LSapp \\
        \bottomrule
    \end{tabular}
\end{table*}

The raw data are processed through the following steps.

\begin{enumerate}
    \item \textbf{Filtering.} Users who have more than 200 distinct apps are removed, leaving 846 users in the Tsinghua dataset and 292 in LSapp. To avoid ambiguity between consecutive usage records of the same app, we merge consecutive records of the same application by the same user into a single entry.\footnote{This means that in the processed data, two consecutive app events always correspond to different apps. This is a necessary compromise because the dataset does not indicate whether multiple consecutive records of the same app belong to a single session or several independent sessions. Among the baselines in Section~\ref{sec:performance}, MFU and MRU are evaluated on this same processed data, ensuring a fair comparison with STAP.}

    \item \textbf{Data splitting.} We randomly divide the users into training, validation, and test sets. For \textbf{cross-dataset transfer} experiments (See in Section~\ref{sec:baseline}), one dataset is split into training and validation sets with an 8:2 ratio, and the entire other dataset serves as the test set. For \textbf{in-dataset} experiments (cold-start scenarios), a single dataset is split into training, validation, and test sets with a 7:1:2 ratio. The exact numbers are listed in Table~\ref{tab:dataset_split}.

    \item \textbf{Segmentation.} The app usage sequences are sliced into segments of length $L/2$ (which will later be transformed into event sequences of length $L$). Segments shorter than $L/2$ (i.e., the remaining part of a sequence whose length is not a multiple of the window size) are padded with placeholder tokens; the corresponding model predictions on these padded positions are masked and excluded from both loss computation and metric evaluation.

    \item \textbf{Feature construction and shuffling.} The sliced app sequences are converted into event sequences, and target labels are generated following Section~\ref{sec:fomulation}. For every sliced sequence that requires encoding, we randomly select a mapping function $\phi$ and apply the shuffling mechanism to obtain $v_i$ and $v_s$ as described in Section~\ref{sec:shuffle}. After this step, the dataset is ready for training and evaluation.
\end{enumerate}

\subsubsection{Evaluation Metrics}\label{sec:metrics}
We adopt two standard metrics for recommendation and sequence prediction tasks: Hit Rate at \(k\) (HR@\(k\)) and Mean Reciprocal Rank at \(k\) (MRR@\(k\))~\cite{TGT,CoSEM}. For a given test sample, HR@\(k\) equals 1 if the correct item appears in the model's top-\(k\) predictions, and 0 otherwise; MRR@\(k\) equals the reciprocal rank of the correct item if it appears in the top \(k\), and 0 otherwise. For a test set of size \(N_{\text{test}}\),
\begin{equation}
    \begin{aligned}
        \text{HR@}k &= \frac{1}{N_{\text{test}}} \sum_{i=1}^{N_{\text{test}}} \mathbf{1}_{\text{hit}_i \leq k},\\
        \text{MRR@}k &= \frac{1}{N_{\text{test}}} \sum_{i=1}^{N_{\text{test}}} \frac{\mathbf{1}_{\text{hit}_i \leq k}}{\text{rank}_i},
    \end{aligned}
\end{equation}
where \(\text{rank}_i\) is the position (starting from 1) of the correct item in the ranked list for the \(i\)-th sample, and \(\mathbf{1}\) denotes the indicator function. %that equals 1 if the correct item is within the top \(k\) predictions, and 0 otherwise. 
In our experiments, we report results for \(k = 1, 3, 5\). All metrics are computed only on the valid positions of each sequence; padded positions are excluded from both loss and metric calculations, as described in the segmentation step in Section~\ref{subsec.process}.

\subsubsection{Implementation Details} \label{sec:implementation}
We employ a standard Transformer architecture as described in Section~\ref{sec:architecture}. The architectural and training hyperparameters are summarized in Table~\ref{tab:hyperparameters}. We always evaluate on the test set using the checkpoint with the lowest validation loss.

\begin{table}[htbp]
    \centering
    \caption{Architectural and Training Hyperparameters}
    \label{tab:hyperparameters}
    \begin{tabular}{lc}
        \toprule
        \textbf{Hyperparameter} & \textbf{Value} \\
        \midrule
        Latent Dimension & 256 \\
        Number of Heads & 4 \\
        Number of Transformer Layers & 8 \\
        Hidden Dimension in FFN & 512 \\
        Base for RoPE & $10^5$ \\
        Optimizer & AdamW \\
        Learning Rate & 0.0003 \\
        Number of Epochs & 1200 \\

        \bottomrule
    \end{tabular}
\end{table}

\subsection{Model Performance} \label{sec:performance}

\subsubsection{Experimental Setup and Evaluation Protocol} \label{sec:baseline}

We evaluate STAP against five baselines: two fundamental methods---\textbf{MRU} (most recently used) and \textbf{MFU} (most frequently used), and three recent leading models with cold-start capabilities—\textbf{CoSEM} \cite{CoSEM}, \textbf{NeuSA} \cite{NeuSA}, and \textbf{MAPLE} \cite{MAPLE}. 
For MRU and MFU, we use the same data preprocessing and sequence segmentation as STAP to ensure a fair comparison. 
CoSEM, NeuSA, and MAPLE rely on fixed app vocabularies and therefore cannot transfer across datasets. For these three models, we report their previously published in-dataset results from the MAPLE paper as reference upper bounds; for across datasets, they are marked as N/A (not applicable).

We adopt two evaluation scenarios that are more challenging than the standard within-dataset temporal split:
\begin{enumerate}
    \item \textbf{Cross-dataset (zero-shot transfer):} Train on one dataset, test on the other. The two datasets contain disjoint users and disjoint app sets, so this measures generalization to completely unseen app ecosystems.
    \item \textbf{In-dataset (cold-start):} Train and test within the same dataset, but the training and test users are disjoint (i.e., user-based partition). This is the standard cold-start scenario in the literature.
\end{enumerate}
The detailed split statistics are given in Table~\ref{tab:dataset_split}.

In each setting, we repeat the experiment 5 times with different random seeds for dataset splitting and parameter initialization; all results are reported as mean ± standard deviation over the 5 runs.

Table~\ref{tab:performance} summarizes the prediction performance under four settings:
\textbf{[Tsinghua $\to$ LSApp]} and \textbf{[LSApp $\to$ Tsinghua]} are the two cross-dataset transfers;
\textbf{[Tsinghua]} and \textbf{[LSapp]} are the in-dataset cold-start splits.
Bold and underlined values indicate the best and second-best results, respectively. N/A denotes that the model has no cross-dataset capability.
For reference, we also include the in-dataset performance of MAPLE as an upper bound, since this model achieved the highest published scores on these benchmarks. 

\textbf{Cross-dataset transfer performance.} CoSEM, NeuSA, and MAPLE rely on fixed app vocabularies and therefore completely lack cross-dataset ability, as indicated by N/A. While the heuristic baselines MRU and MFU can be applied across datasets, STAP consistently outperforms them on all metrics, with particularly large margins on HR@1, MRR@3, and MRR@5. For example, in the [Tsinghua $\to$ LSApp] setting, STAP attains an HR@1 of 68.95\%, whereas the best heuristic baseline (MFU) reaches only 29.14\%. Similar large gains are observed in the reverse direction and on ranking‑based metrics such as MRR@3 and MRR@5.

\textbf{In-dataset cold-start performance.} STAP achieves performance very close to the current state-of-the-art across all metrics (only 1–5\% lower than MAPLE) and significantly outperforms the remaining baselines. 

STAP therefore achieves competitive in‑dataset accuracy while also generalizing effectively to unseen app ecosystems.%\textcolor{blue}{Comments on the results are needed: We should say sth like our approach achieves good performance under both cross- and in-dataset scenarios for $k = 1, 3, 5$ and for both datasets.}

\begin{table*}[htbp]
    \centering
    \caption{Model Performance Comparison( $\%$ )}
    \label{tab:performance}
    \begin{tabular}{lccccc}
        \toprule
        \textbf{scenario} & \multicolumn{5}{c}{\textbf{Cross-dataset [Tsinghua $\to$ LSApp]}} \\
        \cmidrule(lr){2-6}
        \textbf{Model} & \textbf{HR@1} & \textbf{HR@3} & \textbf{HR@5} & \textbf{MRR@3} & \textbf{MRR@5} \\
        \midrule
        MFU & \underline{29.14} & 64.42 & 80.45 & \underline{44.55} & \underline{48.24} \\
        MRU & 0.00 & \underline{79.95} & \underline{87.96} & 38.37 & 40.21 \\
        \midrule
        CoSEM / NeuSA / MAPLE & N/A & N/A & N/A & N/A & N/A \\
        \emph{MAPLE(in-dataset)} & \emph{76.44} & \emph{88.48} & \emph{92.47} & \emph{81.81} & \emph{82.72} \\
        \midrule 
        \textbf{STAP} & \textbf{68.95 ± 0.07} & \textbf{85.19 ± 0.08} & \textbf{90.55 ± 0.05} & \textbf{76.20 ± 0.07} & \textbf{77.43 ± 0.06} \\

        \bottomrule
        \toprule

        \textbf{scenario} & \multicolumn{5}{c}{\textbf{Cross-dataset [LSapp $\to$ Tsinghua]}} \\
        \cmidrule(lr){2-6}
        \textbf{Model} & \textbf{HR@1} & \textbf{HR@3} & \textbf{HR@5} & \textbf{MRR@3} & \textbf{MRR@5} \\
        \midrule
        MFU & \underline{18.55} & 40.74 & 52.54 & \underline{28.23} & \underline{30.92} \\
        MRU & 0.00 & \underline{55.89} & \underline{68.10} & 25.45 & 28.26 \\
        \midrule
        CoSEM / NeuSA / MAPLE & N/A & N/A & N/A & N/A & N/A \\
        \emph{MAPLE(in-dataset)} & \emph{52.28} & \emph{74.17} & \emph{81.28} & \emph{62.06} & \emph{63.69} \\
        \midrule
        \textbf{STAP} & \textbf{42.63 ± 0.28} & \textbf{62.69 ± 0.25} & \textbf{70.65 ± 0.28} & \textbf{51.66 ± 0.26} & \textbf{53.42 ± 0.26} \\

        \bottomrule
        \toprule

        \textbf{scenario} & \multicolumn{5}{c}{\textbf{In-dataset [Tsinghua]}} \\
        \cmidrule(lr){2-6}
        \textbf{Model} & \textbf{HR@1} & \textbf{HR@3} & \textbf{HR@5} & \textbf{MRR@3} & \textbf{MRR@5} \\
        \midrule
        MFU & 18.55 & 40.74 & 52.54 & 28.23 & 30.92 \\
        MRU & 0.00 & 55.89 & 68.10 & 25.45 & 28.26 \\
        \midrule
        CoSEM & 31.11 & 55.97 & 65.25 & 42.04 & 44.16 \\
        NeuSA & 44.33 & 61.69 & 68.12 & 52.06 & 53.53 \\
        MAPLE & \textbf{52.28} & \textbf{74.17} & \textbf{81.28} & \textbf{62.06} & \textbf{63.69} \\
        \midrule
        \textbf{STAP} & \underline{49.99 ± 1.10} & \underline{71.41 ± 1.07} & \underline{78.87 ± 0.94} & \underline{59.52 ± 1.08} & \underline{61.23 ± 1.06} \\

        \bottomrule
        \toprule

        \textbf{scenario} & \multicolumn{5}{c}{\textbf{In-dataset [LSapp]}} \\
        \cmidrule(lr){2-6}
        \textbf{Model} & \textbf{HR@1} & \textbf{HR@3} & \textbf{HR@5} & \textbf{MRR@3} & \textbf{MRR@5} \\
        \midrule
        MFU & 29.14 & 64.42 & 80.45 & 44.55 & 48.24 \\
        MRU & 0.00 & 79.95 & 87.96 & 38.37 & 40.21 \\
        \midrule
        CoSEM & 45.23 & 72.43 & 81.04 & 57.18 & 59.18 \\
        NeuSA & 68.74 & 77.70 & 81.35 & 72.72 & 73.55 \\
        MAPLE & \textbf{76.44} & \textbf{88.48} & \textbf{92.47} & \textbf{81.81} & \textbf{82.72} \\
        \midrule
        \textbf{STAP} & \underline{71.27 ± 2.45} & \underline{86.14 ± 1.84} & \underline{91.05 ± 1.30} & \underline{77.90 ± 2.16} & \underline{79.03 ± 2.02} \\
        
        \bottomrule

    \end{tabular}
\end{table*}

\subsubsection{Ablation Study: Shuffle Mechanism} \label{sec:ablation}

We examine the importance of the shuffle mechanism, focusing on the necessity of re-randomizing the mapping $\phi$ across training epochs. Removing the shuffle mechanism entirely would fix the model to specific app identities, making cross-dataset transfer impossible. We therefore evaluate a weakened variant, termed \textbf{non-shuffle}, in which $\phi$ is randomly initialized for each user but fixed across epochs.

%\textcolor{blue}{Font size of labels in Figure~\ref{fig:ablation_shuffle} should be larger.}
Figure~\ref{fig:ablation_shuffle} compares the training dynamics of the base model and the non-shuffle variant.
The non-shuffle model converges prematurely around epoch~25 and exhibits clear overfitting, whereas the base model continues to improve over more than 1,000 epochs and achieves substantially better validation metrics.
Quantitative cross-dataset results are summarized in Table~\ref{tab:ablation_shuffle}, where ``Degradation'' denotes the absolute performance drop caused by fixing $\phi$ (STAP minus non-shuffle, in percentage points).
The consistent significant drops across all metrics (e.g., HR@1 drops by 6.5~pp, HR@5 by 9.7~pp) confirm that epoch-level re-randomization of the mapping is critical for the regularizing effect of the shuffle mechanism.

\begin{table*}[htbp]
    \centering
    \caption{Ablation results for the shuffle mechanism. Degradation = STAP $-$ non-shuffle (percentage points). All metrics are evaluated under the cross-dataset [Tsinghua $\to$ LSApp] setting.}
    \label{tab:ablation_shuffle}
    \begin{tabular}{lccccc}
        \toprule
        \textbf{Variant} & \textbf{HR@1} & \textbf{HR@3} & \textbf{HR@5} & \textbf{MRR@3} & \textbf{MRR@5} \\
        \midrule
        STAP (base)    & 68.95 & 85.19 & 90.55 & 76.20 & 77.43 \\
        non-shuffle    & 62.12 & 77.59 & 80.90 & 69.17 & 69.93 \\
        \midrule
        Degradation    & 6.47$\downarrow$ & 7.60$\downarrow$ & 9.65$\downarrow$ & 7.03$\downarrow$ & 7.50$\downarrow$ \\
        \bottomrule
    \end{tabular}
\end{table*}

\begin{figure*}[t] 
    \centering
    \includegraphics[width=1.0\textwidth]{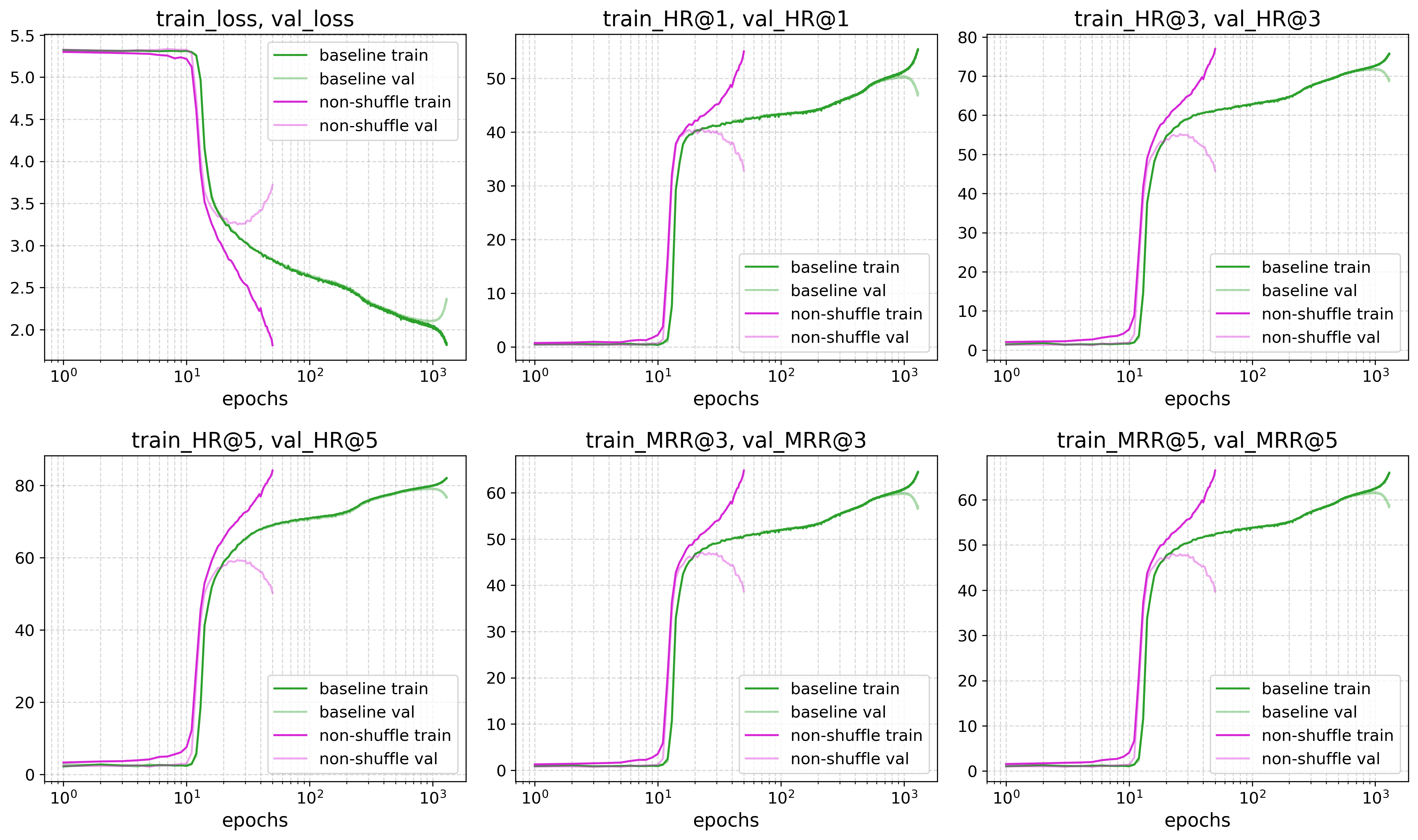} 
    \caption{Training and validation performance comparison between the base model and the non-shuffle ablation variant on the Tsinghua app usage dataset. The subplots illustrate the curves for loss and five evaluation metrics: HR@1, HR@3, HR@5, MRR@3, and MRR@5. The x-axis (number of epochs) is plotted on a logarithmic scale, while the y-axis represents the loss value or metric percentage. It is observed that the base model requires over 1,000 epochs to reach convergence, whereas the non-shuffle version converges prematurely around 25 epochs with significantly degraded performance.}
    \label{fig:ablation_shuffle}
\end{figure*}

\subsubsection{Sensitivity Study: Maximum Context Length}\label{sec:sensitivity}

We evaluate the effect of the maximum context length under the cross-dataset setting [Tsinghua $\to$ LSapp]. 
The context length varies from 16 to 16,384 events (doubling at each step), while all other hyperparameters are held fixed, except for the batch size which is adjusted to keep the total number of tokens per batch unchanged. 
Results are reported in Fig.~\ref{fig:sensitivity_maxlen}.

On both the validation and test sets, performance improves steadily as the context length increases, and the gains become marginal after 4096 events. 
This saturation is consistent with the sequence length distribution in the two datasets: as Fig.~\ref{fig:distribution} shows, most usage sequences contain fewer than 4096 events, so longer contexts provide little additional training signal.
Beyond the benefit plateau, further increasing the context length increases the memory and latency cost of the KV cache without meaningful predictive gains. 
Considering this trade-off, we set $L = 4096$ events as the default context length for all subsequent experiments.

\begin{figure*}[t]
    \centering
    \includegraphics[width=\linewidth]{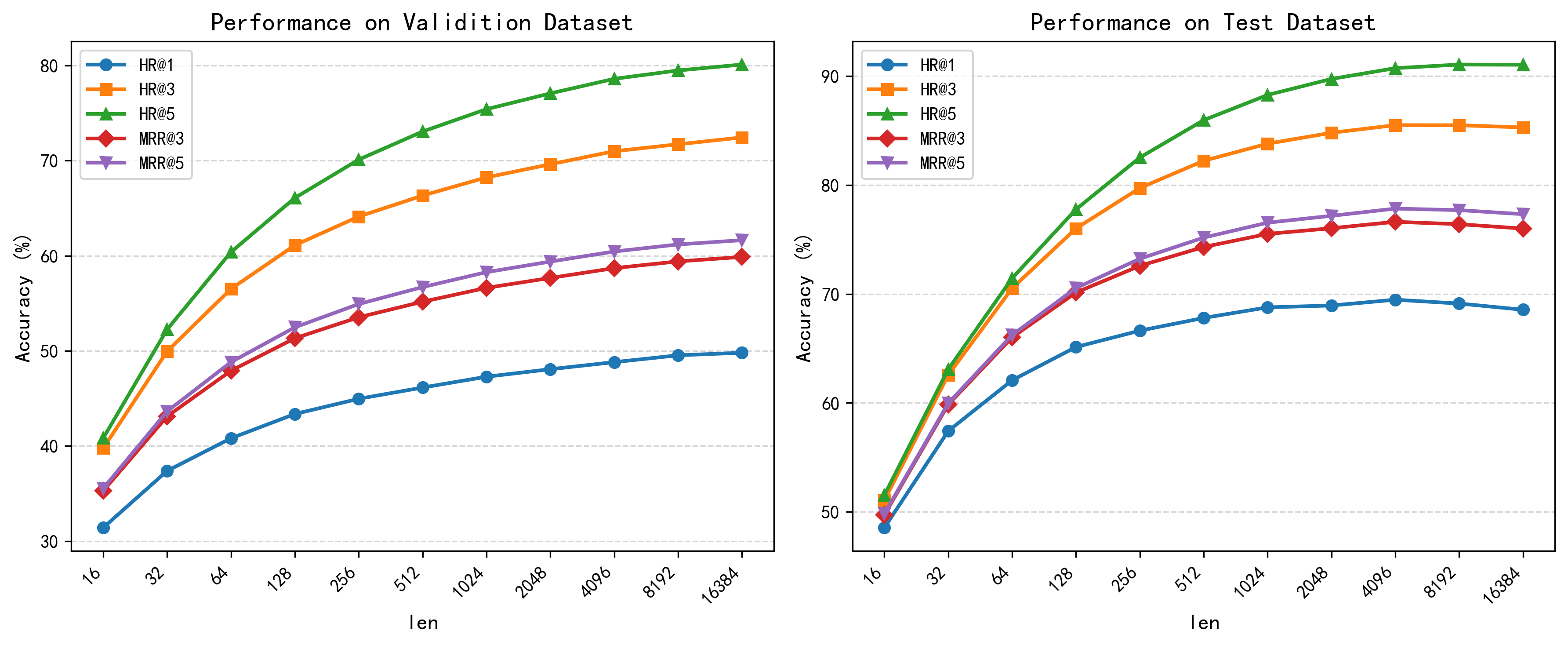} 
    \caption{Impact of the maximum context length on model performance under the cross-dataset setting [Tsinghua $\to$ LSapp]. The x-axis (logarithmic scale) denotes the maximum number of events in the context. The left panel shows metrics on the validation set (Tsinghua), the right panel on the test set (LSapp). All metrics improve as the context length grows, reaching saturation after 4096 events.}
    \label{fig:sensitivity_maxlen}
\end{figure*}

\subsubsection{Computational Efficiency} \label{sec:exp_deploy}

We evaluate the inference latency of STAP in a setting suitable for deployment. 
A lightweight C++ inference engine was implemented following the ISWI strategy (Section~~\ref{sec:iswi}) and executed on a single CPU thread without SIMD or other vectorized acceleration.
Figure~\ref{fig:inference_latency} reports the latency for each event measured over a long session covering more than 3300 app events.

The latency curve exhibits a characteristic sawtooth pattern that directly reflects the ISWI mechanism.
Within each alternation cycle, the latency grows roughly linearly with the length of the total KV cache length of the two instance, consistent with the $\mathcal{O}(L)$ complexity of Transformer decoding for each step.
When the active instance reaches its capacity $L$ and is cleared, the prediction duty shifts to the other instance, whose cache already holds $h = L/2$ events. 
This handover causes an immediate drop in latency, after which a new linear growth phase begins.
As a result, the latency for each step remains bounded—below 50\,ms throughout the entire session—and the long-term average is determined by the steady state cache size, never growing unboundedly.

The total memory footprint of the C++ inference engine is approximately 200\,MB, measured under single-thread execution. 
Of this, 44\,MB is consumed by the model weights (two independent instances with ${\sim}5.5$M parameters each, i.e.~double the single-model weight size). 
The remaining memory is primarily used by the two KV caches and intermediate buffers during inference.

These measurements confirm that ISWI successfully eliminates the startup latency spike of periodic full prefills (which would cost $\mathcal{O}(L^2)$) while preserving the benefits of a deep context.
The single-core C++ implementation keeps per-event latency below 50 ms throughout long sessions, which is well within the response time expected for interactive mobile applications.
Further optimizations such as multi threading of the two inference instances, SIMD acceleration, and FP16 quantization, are expected to provide additional speedup by a factor of several times for on-device deployment under tight resource constraints.

\begin{figure*}[htbp]
    \centering
    \includegraphics[width=0.8\textwidth]{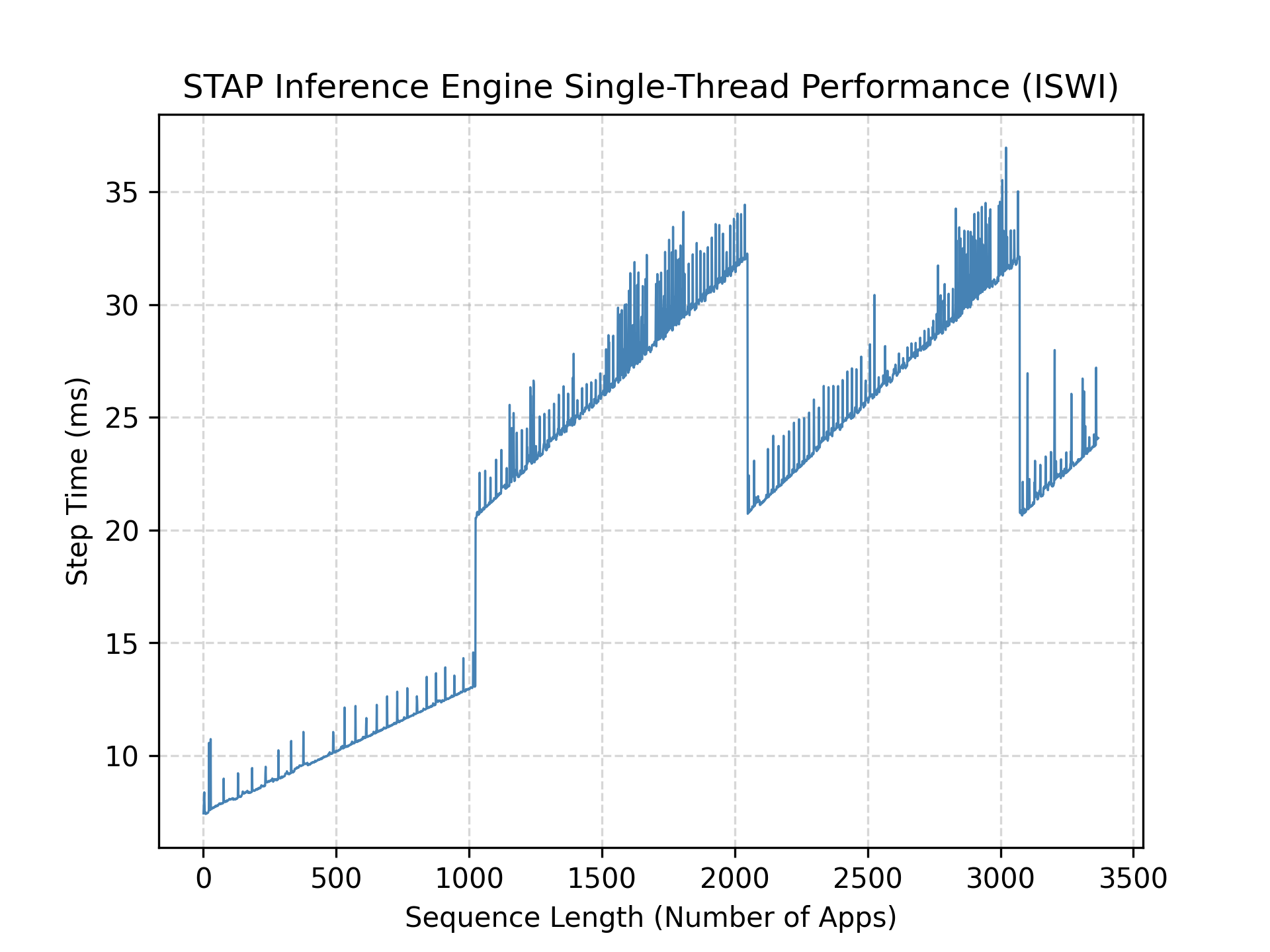}
    
    \caption{Inference latency per event of the STAP C++ engine over a session of more than 3300 app events (single thread, no SIMD). The sawtooth pattern corresponds to the alternating cache resets of ISWI; latency stays below 50\,ms throughout and scales linearly with the total cache length.}
    
    \label{fig:inference_latency}
\end{figure*}

\section{Discussion} \label{sec:discussion}

\subsection{Rationale of the Shuffle Mechanism} \label{sec:theoretical}

The shuffle mechanism replaces true app identities with random virtual indices whose mapping is resampled per user and epoch. This may seem counterintuitive, but its validity can be formally justified.

Let $\mathcal{X}$ be the set of real app IDs and $\mathcal{Y}$ the set of virtual indices, with $|\mathcal{X}| = |\mathcal{Y}| < \infty$.\footnote{If $|\mathcal{X}| < |\mathcal{Y}|$, we can augment $\mathcal{X}$ with $|\mathcal{Y}| - |\mathcal{X}|$ dummy apps that never appear in practice (i.e., their occurrence probability is zero). The resulting augmented set $\mathcal{X}'$ satisfies $|\mathcal{X}'| = |\mathcal{Y}|$, and the original analysis applies directly to $\mathcal{X}'$ and $\mathcal{Y}$.} Denote by $\Phi$ the set of all bijections from $\mathcal{X}$ to $\mathcal{Y}$, and let $\phi$ be a random bijection drawn uniformly from $\Phi$, i.e., $\pi(\phi) = 1/|\Phi|$.
 
% \textcolor{blue}{Note that I replaced all $P$ by ${\rm P}$}

Assume the true app sequence $X_{1:\infty}$ is generated by a stochastic process with conditional distribution ${\rm P}_X$, i.e., ${\rm P}(X_t \mid X_{1:t-1}) = {\rm P}_X(X_t \mid X_{1:t-1})$. Independently, a bijection $\phi^*$ is sampled uniformly from $\Phi$. The observed virtual sequence is then $Y_{1:\infty} = \phi^*(X_{1:\infty})$. The model never observes $X_{1:\infty}$ nor $\phi^*$; it only has access to $Y_{1:t}$ and its task is to predict $Y_{t+1}$.

Given $\phi$, the conditional distribution of $Y_{t+1}$ is
\begin{equation}
    {\rm P}_{\phi}(Y_{t+1}\mid Y_{1:t}) = {\rm P}_X\big(\phi^{-1}(Y_{t+1}) \mid \phi^{-1}(Y_{1:t})\big).
\end{equation}
The optimal Bayesian predictor, which our model aims to approximate with sufficient capacity and data, uses the posterior predictive distribution
\begin{equation}
    {\rm P}(Y_{t+1}\mid Y_{1:t}) = \sum_{\phi\in\Phi} {\rm P}_{\phi}(Y_{t+1}\mid Y_{1:t})\, {\rm P}(\phi\mid Y_{1:t}),
\end{equation}
where ${\rm P}(\phi\mid Y_{1:t})$ is the posterior weight of $\phi$ after observing $Y_{1:t}$.

% \textcolor{blue}{Formalize eq. (7) as a proposition, and the arguments as the proof environment. In the proof, you assume the posterior ${\rm P}(\phi\mid Y_{1:t})$ converges. Under what condition does it converge? Specify it in the proposition.}

\begin{proposition}[Convergence under shuffle]
\label{prop:convergence}
Assume that the posterior distribution \(\mathrm{P}(\phi \mid Y_{1:t})\) converges as $t\to\infty$ and that \(\lim_{t\to\infty} \mathrm{P}(\phi^* \mid Y_{1:\infty}) > 0\). \footnote{In real app usage sequences, user behavior tends to stabilize over time, making the posterior over shuffles converge as more data is observed. The true mapping $\phi^*$ retains positive probability because it generates the observed sequence.} Then the posterior predictive distribution converges to the true distribution under \(\phi^*\):
\begin{equation}
    \lim_{t\to\infty}\big[ \mathrm{P}(Y_{t+1}\mid Y_{1:t}) - \mathrm{P}_{\phi^*}(Y_{t+1}\mid Y_{1:t}) \big] = 0 .
    \label{eq:conv}
\end{equation}
\end{proposition}

\begin{proof}
By Bayes' rule,
\begin{equation}
    {\rm P}(\phi\mid Y_{1:t}) = \frac{\pi(\phi)\,{\rm P}_{\phi}(Y_{1:t})}{\sum_{\phi'}\pi(\phi')\,{\rm P}_{\phi'}(Y_{1:t})},
\end{equation}
and since $\pi(\phi)$ is uniform,
\begin{equation}
    \frac{{\rm P}(\phi\mid Y_{1:t})}{{\rm P}(\phi^*\mid Y_{1:t})}
    = \frac{{\rm P}_{\phi}(Y_{1:t})}{{\rm P}_{\phi^*}(Y_{1:t})}.
\end{equation}
Taking logarithms and using the chain rule,
\begin{equation}
    \log \frac{{\rm P}(\phi\mid Y_{1:t})}{{\rm P}(\phi^*\mid Y_{1:t})}
    = \sum_{i=1}^t \log \frac{{\rm P}_{\phi}(Y_i\mid Y_{1:i-1})}{{\rm P}_{\phi^*}(Y_i\mid Y_{1:i-1})},
    \label{eq:logratio}
\end{equation}
By assumption, $\lim_{t\to\infty} {\rm P}(\phi\mid Y_{1:t}) = c(\phi) \in [0,1]$ exists, and that the true mapping $\phi^*$ retains a nonzero probability ($c(\phi^*) > 0$). For any $\phi$ with $c(\phi) > 0$, the ratio in \eqref{eq:logratio} converges to a finite constant, so the infinite sum of log-ratios is finite. Hence the $i$-th term tends to zero:
\begin{equation}
\lim_{i\to\infty} \log \frac{{\rm P}_{\phi}(Y_i\mid Y_{1:i-1})}{{\rm P}_{\phi^*}(Y_i\mid Y_{1:i-1})} = 0,
\end{equation}
which entails
\begin{equation}
\lim_{i\to\infty} \big| {\rm P}_{\phi^*}(Y_i\mid Y_{1:i-1}) - {\rm P}_{\phi}(Y_i\mid Y_{1:i-1}) \big| = 0. 
\end{equation}
If $c(\phi) = 0$, then simply $\lim_{t\to\infty} {\rm P}(\phi\mid Y_{1:t}) = 0$. Therefore, for every $\phi$,
\begin{equation}
    \lim_{t\to\infty} \big[ {\rm P}_{\phi^*}(Y_{t+1}\mid Y_{1:t}) - {\rm P}_{\phi}(Y_{t+1}\mid Y_{1:t}) \big]\, {\rm P}(\phi\mid Y_{1:t}) = 0.
\end{equation}
Substituting into \eqref{eq:conv} and using the fact that $\Phi$ is finite to exchange sum and limit, we have
\begin{equation}
    \begin{aligned}
        &\lim_{t\to\infty} \big[ {\rm P}(Y_{t+1}\mid Y_{1:t}) - {\rm P}_{\phi^*}(Y_{t+1}\mid Y_{1:t}) \big] \\
        =& \sum_{\phi} \lim_{t\to\infty} \big[ {\rm P}_{\phi}(Y_{t+1}\mid Y_{1:t}) - {\rm P}_{\phi^*}(Y_{t+1}\mid Y_{1:t}) \big]\, {\rm P}(\phi\mid Y_{1:t}) \\
        =& 0.
    \end{aligned}
\end{equation}

\end{proof}

In summary, with a sufficiently long context, either the true mapping is uniquely identified, or all remaining possible mappings eventually yield the same predictions, making the uncertainty irrelevant. In practice, distinct real app sequences rarely produce structurally indistinguishable virtual sequences over long horizons, so the first situation dominates. This shows that the shuffle mechanism is valid as long as the context is sufficiently long, which is  the condition of our ultra-long context design.

\begin{figure*}[t]
    \centering
    \includegraphics[width=\linewidth]{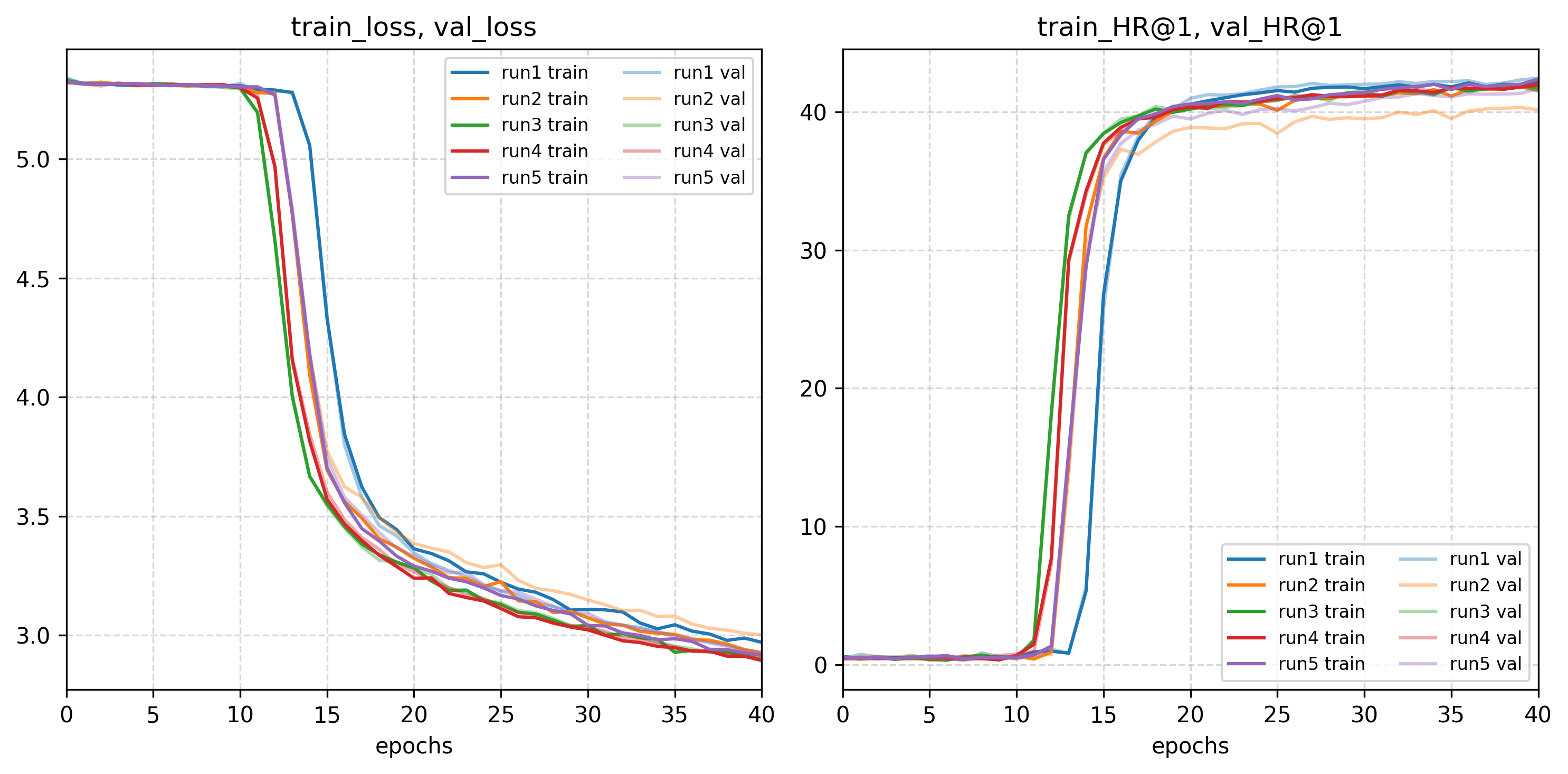}
    \caption{Training loss and HR@1 curves of the baseline model on the Tsinghua App Usage Dataset over five runs with different random seeds. The model performs close to random prediction in the first 10 epochs. Around epochs 11--15, the performance improves sharply, and then the model enters a more stable training phase with gradual gains.}
    \label{fig:grok1}
\end{figure*}

\begin{figure*}[t]
    \centering
    \includegraphics[width=\linewidth]{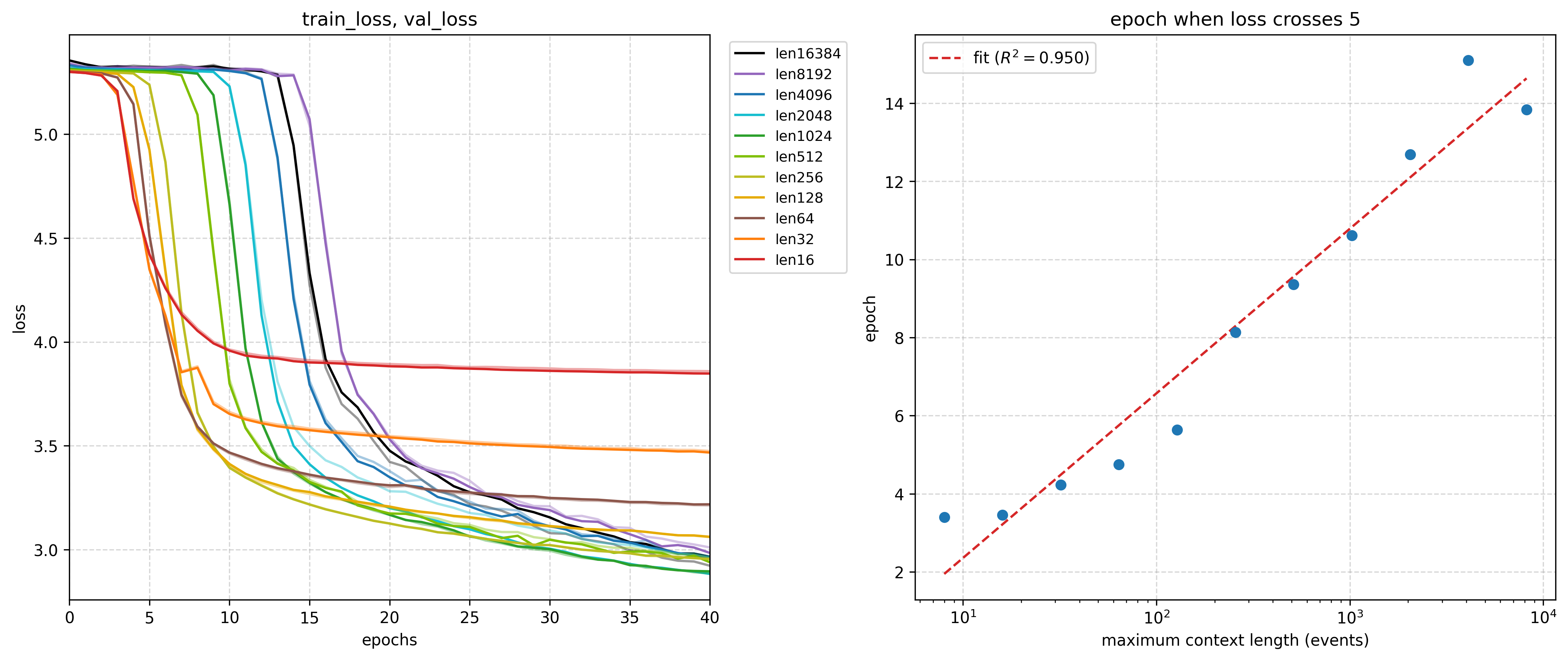}
    \caption{Relation between the phenomenon epoch and the maximum context length (in events). The left panel shows the loss curves for different context lengths in the sensitivity study of Section~~\ref{sec:sensitivity}. The numbers in the labels (e.g., len4096) refer to the maximum number of events in the context. The right panel shows the epoch when the loss first reaches 5.0 for each context length, estimated by linear interpolation between adjacent epochs. The fitted line on the log-transformed context lengths yields $R^2=0.95$, confirming a strong log-linear relationship.}
    \label{fig:grok2}
\end{figure*}

\subsection{Observations on Training Dynamics} \label{sec:grok}

We report a notable observation on the training dynamics of STAP. 
Unlike typical training curves, where the loss drops rapidly at the beginning and then gradually converges, STAP exhibits a prolonged flat phase during the early epochs, with prediction metrics close to random guessing. 
After a number of epochs, the loss suddenly drops and the accuracy increases sharply within a few epochs.

Taking the baseline configuration from Section~\ref{sec:baseline} as an example (Fig.~\ref{fig:grok1}), this phenomenon consistently appears around epochs 11--15 across multiple runs on the Tsinghua App Usage Dataset. 
During the first $\sim$10 epochs, the loss fluctuates mildly between 5.31 and 5.35, and HR@1 remains near $0.5\%$. 
Once the loss falls below 5.30, it drops rapidly to $3.5$--$4.0$ within roughly 5 epochs, while HR@1 increases abruptly to approximately $40\%$. 
Thereafter, training follows a standard pattern: the loss continues to decline slowly, and the accuracy keeps improving but at a diminishing rate. 
This stable phase lasts over 1,000 epochs before any sign of overfitting appears.

Furthermore, in the sensitivity analysis of the maximum context length (Section~\ref{sec:sensitivity}), we observed a clear log-linear relationship: the {\it phenomenon epoch} scales linearly with the logarithm of the context length (Fig.~\ref{fig:grok2}). 
We define the phenomenon epoch as the first epoch when the training loss reaches $5.0$, estimated by linear interpolation between adjacent recorded epochs. 
A linear fit on the log-transformed context lengths yields $R^2 = 0.95$, indicating that the position of this sudden learning phase is strongly correlated with the size of the context window.

\subsection{Limitations and Future Work}\label{sec:future}

Despite the strong cross-dataset generalization and efficient deployment achieved by STAP, several aspects remain open for investigation.

First, the virtual vocabulary size $V$ is currently fixed at 200. While this covers the vast majority of users, it excludes a small fraction who have installed and used more than 200 distinct applications. An adaptive mechanism that adjusts $V$ per user, or a hierarchical indexing approach, could broaden the applicability without sacrificing the core shuffle principle.

Second, the sensitivity analysis in Section~\ref{sec:sensitivity} shows that increasing the context length beyond $4096$ events yields diminishing prediction accuracy. However, both datasets contain relatively few sequences long enough to benefit from larger windows. It therefore remains unclear whether the observed saturation reflects a fundamental limitation of the method or simply data scarcity. Larger-scale datasets containing users with longer sequences would help answer this question.

Third, the current evaluation is limited to two datasets collected in different regions but within a similar time period. 
The Tsinghua App Usage Dataset was collected in Shanghai, China in 2016 \cite{TsinghuaApp}, while the LSapp dataset was collected in North America during 2017--2018 \cite{NeuSA}. 
Cross-dataset generalization across substantially different time spans (e.g., training on data from one decade and testing on another) has not been examined. Extending the evaluation to more diverse ecosystems would further validate the universality of the shuffle mechanism.

On the theoretical side, the prolonged flat phase followed by rapid generalization observed in Section~\ref{sec:grok} remains without a mechanistic explanation. Two directions worth further study: understanding how different hyperparameters influence this phenomenon, and uncovering the mechanism behind it. Both would help reveal how Transformers behave under extreme vocabulary anonymization and better understand their capability boundaries.

Finally, a practical aspect worth highlighting is that STAP requires no user identifiers or app metadata, making it naturally suited for on-device inference where privacy regulations are strict. Future engineering efforts could further reduce the memory footprint through weight sharing between the two ISWI instances and quantization, bringing the model closer to deployment on low-end devices.
\section{Conclusion}
\label{sec:conclusion}

This paper studied next-app prediction under the constraint of vocabulary independence, with the goal of enabling models that generalize across different app ecosystems without retraining.
We proposed STAP, which integrates two complementary mechanisms: a shuffle mechanism that replaces true app identities with randomly reassigned virtual indices, and an ultra-long context design that allows the model to infer app semantics purely from extended temporal and structural patterns. 
To resolve the conflict between long-context dependence and bounded inference cost, we introduced ISWI, a deployment strategy that guarantees a minimum context of $L/2$ at every prediction step while ensuring an acceptable upper bound on latency.

Experiments on two datasets collected from different continents demonstrate that STAP performs cross-dataset zero-shot prediction at practical accuracy levels (e.g., HR@1 up to 69\% in [Tsinghua $\to$ LSapp]), whereas all existing fixed-vocabulary methods are inherently inapplicable, while in-dataset cold-start performance is competitive with state-of-the-art models. Meanwhile, A single-core C++ implementation of STAP achieves per-event latency below 50\,ms and a total memory footprint of 200\,MB, confirming its feasibility for real-time on-device inference.

Ablation experiments verified that re-randomizing the virtual mapping across epochs is essential to prevent overfitting and to realize the regularizing effect of the shuffle mechanism, while sensitivity analysis showed that prediction accuracy increases substantially with longer context windows before gradually saturating, confirming the value of the ultra-long context design and indicating that the chosen length captures most of the achievable benefit. 

Meanwhile, a theoretical analysis showed that even when the random mapping is not uniquely identifiable, the Bayesian predictive distribution still converges to the correct one given a sufficiently long context, providing formal justification for the shuffle design.

We also reported a noteworthy training phenomenon: the model remains near-random for many epochs before the loss suddenly drops and accuracy spikes, with the onset epoch exhibiting a clear log-linear relationship to the context length.

These results show that fixed app identities are not a strictly necessary for accurate next-app prediction: shuffled virtual tokens, combined with sufficiently long context, can support competitive accuracy while additionally enabling zero-shot transfer across app ecosystems.
The vocabulary-free nature of STAP also makes it naturally suited for privacy-sensitive deployment scenarios.

Future work includes training on larger-scale datasets with richer long sequence usage records and evaluating cross-dataset generalization across significantly longer time spans, as well as developing theoretical understandings of the distinctive training dynamics observed in ultra-long shuffled contexts.

\bibliographystyle{elsarticle-num}

\bibliography{cas-refs}

\end{document}